\documentclass[11pt]{article}

\usepackage[margin=1in]{geometry}

\usepackage{amsmath}
\usepackage{amssymb}
\usepackage{amsthm}
\usepackage{mathtools}

\usepackage{booktabs}
\usepackage{array}

\usepackage{listings}
\usepackage{xcolor}

\usepackage[colorlinks=true,linkcolor=blue,citecolor=blue,urlcolor=blue]{hyperref}

\usepackage{natbib}

\usepackage{enumitem}
\usepackage{microtype}

\newtheorem{theorem}{Theorem}[section]

\newtheorem{corollary}[theorem]{Corollary}

\theoremstyle{definition}

\theoremstyle{remark}
\newtheorem{remark}[theorem]{Remark}

\definecolor{codebg}{rgb}{0.95,0.95,0.95}
\definecolor{codegreen}{rgb}{0,0.6,0}
\definecolor{codeblue}{rgb}{0.0,0.0,0.7}

\lstdefinestyle{coq}{
  backgroundcolor=\color{codebg},
  basicstyle=\ttfamily\small,
  breaklines=true,
  frame=single,
  framerule=0pt,
  xleftmargin=1em,
  xrightmargin=1em,
  aboveskip=0.5em,
  belowskip=0.5em,
  keywordstyle=\color{codeblue}\bfseries,
  commentstyle=\color{gray}\itshape,
  stringstyle=\color{codegreen},
  morekeywords={Definition,Theorem,Lemma,Corollary,Proof,Qed,
    Inductive,Record,Fixpoint,CoFixpoint,Section,End,Variable,
    Hypothesis,forall,exists,fun,match,with,end,if,then,else,
    let,in,return,Type,Prop,Set,Variant},
  morecomment=[s]{(*}{*)},
  morestring=[b]",
  sensitive=true,
  showstringspaces=false,
  literate={/\\}{{/\textbackslash}}2 {\\\/}{{\textbackslash/}}2,
}

\newcommand{\Gov}{\mathcal{G}}
\newcommand{\code}{\textsc{code}}
\newcommand{\reason}{\textsc{reason}}
\newcommand{\memory}{\textsc{memory}}
\newcommand{\mcall}{\textsc{call}}
\newcommand{\govsafe}{\texttt{gov\_safe}}
\newcommand{\itree}{\mathrm{itree}}
\newcommand{\DirectiveE}{\mathrm{DirectiveE}}
\newcommand{\IOE}{\mathrm{IOE}}
\newcommand{\GovIOE}{\mathrm{GovIOE}}
\newcommand{\GovE}{\mathrm{GovE}}
\newcommand{\Ret}{\mathrm{Ret}}
\newcommand{\Tau}{\mathrm{Tau}}
\newcommand{\Vis}{\mathrm{Vis}}
\newcommand{\bind}{\mathbin{>\!\!>\!\!=}}
\newcommand{\interp}{\mathrm{interp}}

\begin{document}

\title{Effect-Transparent Governance for AI Workflow Architectures:\\
Semantic Preservation, Expressive Minimality, and Decidability Boundaries}

\author{Alan L. McCann\\
\textit{Mashin, Inc.}\\
\texttt{research@mashin.live}}

\date{April 2026}
\maketitle

% ==================================================================
\begin{abstract}
We present a machine-checked formalization of structurally governed AI
workflow architectures and prove that effect-level governance can be
imposed without reducing internal computational expressivity. Using
Interaction Trees in Rocq~8.19, we define a governance operator $\Gov$
that mediates all effectful directives, including memory access,
external calls, and oracle (LLM) queries. Our development compiles with
\textbf{0 admitted lemmas} and consists of \textbf{36 modules},
\textbf{~12,000 lines} of Rocq, and \textbf{454 theorems}.

We establish seven properties: (P1) governed Turing completeness,
(P2) governed oracle expressivity, (P3) a decidability boundary in
which governance predicates are total and closed under Boolean
composition while semantic program properties remain non-trivial and
undecidable by governance, (P4) goal preservation for permitted
executions, (P5) expressive minimality of primitive capabilities
(compute, memory, reasoning, external call, observability), (P6)
subsumption asymmetry showing structural governance strictly subsumes
content-level filtering, and (P7) \emph{semantic transparency}: on all
executions where governance permits, the governed interpretation is
observationally equivalent (modulo governance-only events) to the
ungoverned interpretation. Together, these results show that governance
and computational expressivity are orthogonal dimensions: governance
constrains the effect boundary of programs while remaining semantically
transparent to internal computation.
\end{abstract}

% ==================================================================
\section{Introduction}
\label{sec:intro}
% ==================================================================

AI workflow systems increasingly act in the world: they query models,
store and retrieve state, call external tools, and coordinate
distributed services. As these systems become more capable, governance
mechanisms are layered onto them to prevent unsafe or unauthorized
behavior. A widely held assumption is that governance inevitably trades
against capability: restrictions intended to increase safety reduce the
space of behaviors the system can exhibit.

This paper formalizes and mechanically verifies an architectural
alternative: \emph{effect-level structural governance}, extending our
prior mechanized development~\cite{mccann2026mechanized} with semantic
transparency, expressive minimality, and decidability boundary results.
Instead of modifying internal computation or post-processing outputs,
governance is imposed at the effect boundary of the system. Programs are modeled
as interaction trees over a directive event type. Governance is a
wrapper over event interpretation in the tradition of algebraic effect
handlers~\cite{plotkin2009handlers}: it mediates each directive with
structural checks (e.g.\ trust and capability predicates) and
governance-only instrumentation (e.g.\ provenance and observability).

The central semantic question is not merely whether governed executions
are safe, but whether governance changes what permitted programs
compute. Our main contribution is a machine-checked theorem that it
does not.

\paragraph{Main Results (informal).}
We prove two foundational properties and five supporting results.

\begin{itemize}[leftmargin=2em]
\item \textbf{Safety:} governed interpretation enforces effect safety
  for all programs and handlers.
\item \textbf{Semantic transparency:} for executions in which governance
  checks succeed, the governed interpretation is observationally
  equivalent (modulo governance-only events) to the ungoverned
  interpretation. In particular, permitted programs compute the same
  partial function and produce the same reachable return values.
\end{itemize}

We further establish governed expressivity (Turing completeness and an
oracle/LLM extension), an explicit decidability boundary and
compositional closure for governance predicates, expressive minimality
of primitive capabilities, and a strict subsumption separation between
structural governance and content-level filtering.

\paragraph{Artifact.}
The complete development compiles in Rocq~8.19 with Interaction Trees,
comprising 36 modules, ~12,000 lines, 454 theorems, and 0 admitted
lemmas. The artifact is available at
\url{https://github.com/mashin-live/governance-proofs}.

\paragraph{Contributions.}
\begin{enumerate}[leftmargin=2em]
\item A machine-checked semantics for effect-governed AI workflows
  using Interaction Trees.
\item A semantic transparency theorem: governed execution is equivalent
  to ungoverned execution on all permitted traces (modulo
  governance-only events).
\item Governed expressive sufficiency: Turing completeness and oracle
  extension under governance.
\item Expressive minimality: each primitive capability is structurally
  necessary; removing any primitive strictly reduces expressiveness.
\item A decidability boundary with compositional closure for governance
  predicates and a separation showing governance cannot decide halting.
\item A strict subsumption theorem: structural governance subsumes
  content filtering; content filtering does not imply structural safety.
\end{enumerate}

\paragraph{Scope and companion papers.}
This paper focuses on expressiveness and semantic transparency results. It builds on the mechanized safety and invariance proofs in~\cite{mccann2026mechanized} and extends the structural governance criterion introduced in~\cite{mccann2026structural}.
The algebraic semantics paper~\cite{mccann2026algebraic} lifts these system-specific results to a parametric framework, showing that any system satisfying three axioms inherits the safety and transparency properties proved here.
The certified purity paper~\cite{mccann2026purity} addresses the practical enforcement of the pure module constraint that the safety theorems assume.
The distribution provenance paper~\cite{mccann2026provenance} extends governance to the supply chain with cryptographic provenance.

% ==================================================================
\section{Motivating Example: Governing Effects Without Altering Computation}
\label{sec:motivating}
% ==================================================================

Consider a workflow that generates a plan via an LLM, stores the plan,
invokes an external tool, and returns a result. Schematically:

\begin{lstlisting}[style=coq]
ITree.bind (trigger (LLMCall prompt)) (fun plan =>
ITree.bind (trigger (MemStore key plan)) (fun _ =>
ITree.bind (trigger (CallMachine api plan)) (fun resp =>
Ret resp)))
\end{lstlisting}

This workflow combines internal computation with effectful operations:
oracle inference (\texttt{LLMCall}), state (\texttt{MemStore}), and
external action (\texttt{CallMachine}). Content-level governance often
operates by modifying outputs (filters, rewriting, alignment prompts),
which can change the observable behavior of the system even when
execution is allowed.

Structural governance instead wraps effect interpretation. For any base
handler $h$, the governed handler $\Gov(h)$ mediates every directive:
it checks governance predicates and either delegates to $h$ or denies
execution by divergence. This yields two key observations:

\begin{itemize}[leftmargin=2em]
\item \textbf{Effect safety:} unauthorized effects cannot occur because
  all directives are mediated.
\item \textbf{Conditional semantic preservation:} if governance permits
  the execution, the return value is unchanged (modulo
  governance-only events).
\end{itemize}

The remainder of the paper makes these observations precise and proves
them for all programs expressible in the architecture.

% ==================================================================
\section{Background}
\label{sec:background}
% ==================================================================

We recall the interaction-tree model and the governance wrapper
construction, introduced in~\cite{mccann2026mechanized}. All
definitions are mechanized in Rocq using Interaction
Trees~\cite{xia2020interaction,zakowski2021itree}.

\subsection{Interaction Trees}

Programs are coinductive trees over an event type $E$:
\[
  \itree\ E\ R ::= \Ret(v) \mid \Tau(t) \mid \Vis(e, k)
\]
where $e : E\;S$ and $k : S \to \itree\ E\ R$.

\subsection{Directives and Governance Operator}

The directive event type $\DirectiveE$ contains constructors for
primitive capabilities: computation morphisms, memory operations,
oracle calls, external tool calls, and observability events. A handler
interprets directives into an effectful semantic domain.

The governance operator $\Gov$ wraps directive interpretation with
pre-checks and post-instrumentation, denying execution on failed
checks. This is handler-level effect
mediation~\cite{plotkin2009handlers,bauer2015programming}: the handler
structure is preserved, with governance interposed at interpretation
time:

\[
  \Gov(h)(d) = \text{pre\_checks} \bind
    \lambda b.\; \text{if } b \text{ then } h(d) \bind
    \lambda r.\; \text{post\_record} \bind
    \lambda \_.\; \Ret(r)
    \text{ else spin}.
\]

\begin{lstlisting}[style=coq]
Definition Gov (h : base_handler) : governed_handler :=
  fun R (d : DirectiveE R) =>
    ITree.bind pre_governance (fun ok =>
    if ok then
      ITree.bind (lift_io (h R d)) (fun r =>
      ITree.bind post_governance (fun _ =>
      ret r))
    else
      ITree.spin).
\end{lstlisting}

\noindent Here \texttt{pre\_governance} emits four $\GovE$ events
(trust check, permission check, phase validation, pre-hooks) that each
return a boolean, short-circuiting on failure.
\texttt{post\_governance} emits three $\GovE$ events (guardrails,
provenance record, event broadcast). \texttt{spin} is infinite silent
divergence. The type $\GovIOE = \GovE +' \IOE$ separates
governance-only events from real I/O events.

\subsection{Safety Predicate}

We express safety via a coinductive predicate $\govsafe$ over governed
executions, defined using parameterized
coinduction~\cite{hur2013power}. Informally,
$\govsafe(\text{false},t)$ states that $t$ does not perform
unauthorized effects; $\govsafe(\text{true},t)$ tracks the permitted
state.

\begin{lstlisting}[style=coq]
Variant gov_safeF (F : bool -> itree GovIOE R -> Prop)
    : bool -> itreeF GovIOE R (itree GovIOE R) -> Prop :=
  | GS_Ret  : forall allowed r,
      gov_safeF F allowed (RetF r)
  | GS_Tau  : forall allowed t,
      F allowed t -> gov_safeF F allowed (TauF t)
  | GS_GovE : forall allowed (s : GovernanceStage) k,
      (forall b, F true (k b)) ->
      gov_safeF F allowed (VisF (inl1 (GovCheck s)) k)
  | GS_IOE  : forall (e : IOE X) (k : X -> itree GovIOE R),
      (forall x, F true (k x)) ->
      gov_safeF F true (VisF (inr1 e) k).

Definition gov_safe := paco2 gov_safe_ bot2.
\end{lstlisting}

\noindent The key asymmetry: \texttt{GS\_GovE} applies at
\emph{any} permission level (governance events grant permission),
while \texttt{GS\_IOE} requires $\text{allowed} = \text{true}$
(I/O requires prior authorization). Bare I/O at $\text{false}$
has no matching constructor and is therefore unsafe.

% ==================================================================
\section{Safety and Semantic Transparency}
\label{sec:transparency}
% ==================================================================

This section presents the two foundational results: universal effect
safety and semantic transparency on permitted executions.

\subsection{Universal Effect Safety}
\label{sec:safety}

\begin{theorem}[Governed Interpretation Safety]
\label{thm:governed-safe}
For any handler $h$ and any program $t : \itree\ \DirectiveE\ R$,
\[
  \govsafe(\text{false},\; \interp(\Gov(h), t)).
\]
\end{theorem}

\noindent Theorem~\ref{thm:governed-safe} is uniform over directive
types: it applies equally to memory, oracle calls, tool calls, and
observability events.

\subsection{Semantic Transparency on Permitted Executions}

We define a permissiveness condition capturing that governance checks
succeed along an execution. Under this condition, governance is
observationally transparent modulo governance-only events.

\begin{lstlisting}[style=coq]
Definition permissive_gov : forall X, GovE X -> X :=
  fun X e => match e with GovCheck _ => true end.

Definition permissive_handler (io_h : forall X, IOE X -> X)
  : forall X, GovIOE X -> X :=
  fun X e =>
    match e with
    | inl1 ge => permissive_gov X ge
    | inr1 ie => io_h X ie
    end.
\end{lstlisting}

\noindent The permissive handler resolves every governance check to
$\text{true}$ and delegates I/O events to an arbitrary handler
$\text{io\_h}$. This models the case where governance permits
everything. The semantic question: when governance permits, does it
alter the result?

\begin{lstlisting}[style=coq]
Theorem governed_transparency :
  forall (h : base_handler) (io_h : forall X, IOE X -> X)
         R (t : itree DirectiveE R),
    eutt eq
      (interp (fun X e => Ret (permissive_handler io_h X e))
              (interp (Gov h) t))
      (interp (fun X e => Ret (io_h X e))
              (interp h t)).
\end{lstlisting}

\begin{theorem}[Semantic Transparency (informal)]
\label{thm:semantic-transparency-informal}
On all executions where governance permits, the governed interpretation
is observationally equivalent (modulo governance-only events) to the
ungoverned interpretation.
\end{theorem}

\paragraph{Observational equivalence.}
The equivalence relation \texttt{eutt eq} denotes weak bisimulation up
to silent transitions ($\Tau$). Governance-only events are interpreted
by the permissive handler and do not alter observable I/O behavior.
Consequently, semantic transparency establishes equality of the
computed partial function on all executions where governance checks
succeed. On the left side, the permissive handler interprets
governance-only $\GovE$ events to $\text{true}$ and erases them; on the
right side, no governance events exist. What remains after erasure is
identical: the same I/O events, in the same order, producing the same
return values.

\begin{lstlisting}[style=coq]
Corollary governed_same_result :
  forall h io_h R (t : itree DirectiveE R) (v : R),
    eutt eq (interp (fun X e => Ret (io_h X e))
                    (interp h t)) (Ret v) ->
    eutt eq (interp (fun X e => Ret (permissive_handler io_h X e))
                    (interp (Gov h) t)) (Ret v).

Corollary governed_goal_preservation :
  forall h io_h R (G : R -> Prop) t (v : R),
    G v ->
    eutt eq (interp (fun X e => Ret (io_h X e))
                    (interp h t)) (Ret v) ->
    eutt eq (interp (fun X e => Ret (permissive_handler io_h X e))
                    (interp (Gov h) t)) (Ret v).
\end{lstlisting}

\paragraph{Denial trades liveness for safety.}
Governance preserves safety by converting unauthorized executions into
divergence. When a governance check fails, the wrapped directive
evaluates to \texttt{ITree.spin}, producing no observable result. The
artifact proves that spin is always safe (\texttt{spin\_gov\_safe}),
that binding a continuation to a divergent tree preserves safety
(\texttt{bind\_spin\_gov\_safe}), and that spin never equals a return
value (\texttt{spin\_not\_ret}):

\begin{lstlisting}[style=coq]
Lemma spin_gov_safe :
  forall allowed, gov_safe allowed ITree.spin.

Theorem spin_not_ret :
  forall R (v : R),
    ~ eq_itree eq ITree.spin (Ret v).
\end{lstlisting}

\noindent The system therefore trades liveness (possible termination)
for safety (absence of unauthorized effects), but never produces an
incorrect return value. Transparency holds only on permitted
executions; denial yields divergence rather than semantic equivalence.

% ==================================================================
\section{Governed Expressivity: Turing and Oracle Completeness}
\label{sec:expressivity}
% ==================================================================

We encode register machines~\cite{minsky1967computation} to establish
Turing completeness and extend the encoding with oracle queries to
model LLM integration.

\subsection{Governed Turing Completeness}

\begin{theorem}[Governed Turing Completeness]
\label{thm:turing}
For any handler $h$, register-machine program $p$, and fuel bound,
\[
  \govsafe(\text{false},\;
    \interp(\Gov(h),\; \texttt{translate\_program}(p,\text{fuel},0))).
\]
\end{theorem}

\begin{proof}
Directly from Theorem~\ref{thm:governed-safe}, since the translation
produces an $\itree$ over $\DirectiveE$.
\end{proof}

\subsection{Governed Oracle Expressivity}

We extend the machine with an oracle instruction \texttt{O\_QUERY} that
emits an \texttt{LLMCall} directive.

\begin{theorem}[Governed Oracle Expressivity]
\label{thm:oracle}
For any handler $h$, oracle program $p$, fuel bound, and program counter,
\[
  \govsafe(\text{false},\;
    \interp(\Gov(h),\; \texttt{translate\_oracle\_program}(p,\text{fuel},\text{pc}))).
\]
\end{theorem}

\begin{proof}
Again by Theorem~\ref{thm:governed-safe}.
\end{proof}

% ==================================================================
\section{Decidability Boundary and Compositional Closure}
\label{sec:decidability}
% ==================================================================

Governance predicates are structural: they are total, syntactic, and
closed under composition. In contrast, semantic properties of programs
(e.g.\ halting) remain non-trivial and cannot be decided by
governance. This formalizes the boundary concept
from~\cite{mccann2026structural}.

\subsection{Governance Policies and Closure}

We model governance policies as inductive syntax with a total evaluator.
We prove closure under Boolean composition (AND/OR/NOT).

\begin{lstlisting}[style=coq]
Inductive GovPolicy : Type :=
  | PolCapability : TrustLevel -> Capability
                      -> list Capability -> GovPolicy
  | PolTrust : TrustLevel -> TrustLevel -> GovPolicy
  | PolAnd   : GovPolicy -> GovPolicy -> GovPolicy
  | PolOr    : GovPolicy -> GovPolicy -> GovPolicy
  | PolNot   : GovPolicy -> GovPolicy.

Fixpoint eval_policy (pol : GovPolicy) : bool :=
  match pol with
  | PolCapability tl cap dcaps => capability_allowed tl cap dcaps
  | PolTrust tl1 tl2 => trust_at_least tl1 tl2
  | PolAnd p1 p2     => eval_policy p1 && eval_policy p2
  | PolOr  p1 p2     => eval_policy p1 || eval_policy p2
  | PolNot p         => negb (eval_policy p)
  end.

Theorem compositional_closure :
  (forall p1 p2, eval_policy (PolAnd p1 p2) = true \/
                 eval_policy (PolAnd p1 p2) = false) /\
  (forall p1 p2, eval_policy (PolOr p1 p2) = true \/
                 eval_policy (PolOr p1 p2) = false) /\
  (forall p, eval_policy (PolNot p) = true \/
             eval_policy (PolNot p) = false).
\end{lstlisting}

\subsection{Governance Cannot Decide Halting}

We prove a separation: there exist programs that are indistinguishable
to governance predicates (same directive class/shape) but differ in
halting behavior; thus governance cannot decide halting.

\begin{lstlisting}[style=coq]
Theorem governance_cannot_decide_halting :
  (* Both programs' first instruction emits MemoryOp *)
  (exists params k,
    translate_instruction (INC 0 1) =
    ITree.bind (trigger (MemoryOp params)) k) /\
  (exists params k,
    translate_instruction (INC 0 0) =
    ITree.bind (trigger (MemoryOp params)) k) /\
  (* But one halts and one diverges *)
  rm_halts inc_then_halt 2 (0, init_regs) = true /\
  (forall n, rm_halts looping_program n (0, init_regs) = false).
\end{lstlisting}

\noindent Governance predicates operate on the structural properties of
individual directives (event type, trust level, capabilities) and not on
recursive control flow. Because two programs can present identical
directive structure on their first step yet diverge in global behavior,
no total governance predicate can decide termination. This is a
consequence of Rice's theorem~\cite{rogers1967theory} applied to the
decidability boundary: governance correctly restricts itself to the
structurally decidable side.

\noindent Together, these results establish a structural decidability
boundary: governance predicates are total and compositional, while
semantic program properties remain undecidable.

% ==================================================================
\section{Expressive Minimality of Primitive Capabilities}
\label{sec:minimality}
% ==================================================================

Surjection of primitives onto a capability taxonomy is not sufficient
to rule out redundancy. We therefore prove \emph{expressive minimality}:
each primitive capability is structurally necessary for full
expressiveness.

We classify directive events into disjoint capability classes and show
strict separations:

\begin{itemize}[leftmargin=2em]
\item \textbf{Memory necessity:} register-machine simulation requires
  memory directives.
\item \textbf{Reason necessity:} oracle extension requires \texttt{LLMCall}.
\item \textbf{Call necessity:} external action morphisms require tool-call directives.
\item \textbf{Hierarchy:} \code{} $<$ \memory{} $<$ \reason{} $<$ \mcall{} (strict).
\end{itemize}

\begin{lstlisting}[style=coq]
Definition is_memory_event {R} (d : DirectiveE R) : bool :=
  match d with
  | MemoryOp _ | DBOp _ | FileOp _ => true
  | _ => false
  end.
Definition is_reason_event {R} (d : DirectiveE R) : bool :=
  match d with
  | LLMCall _ | LLMCallStream _ => true
  | _ => false
  end.
Definition is_call_event {R} (d : DirectiveE R) : bool :=
  match d with
  | CallMachine _ | HTTPRequest _ | ExecOp _
  | GraphQLRequest _ | WebSocketOp _ | MCPCall _ => true
  | _ => false
  end.

Lemma event_classes_disjoint : forall R (d : DirectiveE R),
  (* each event belongs to exactly one class *)
  (is_memory_event d = true /\ is_reason_event d = false
   /\ is_call_event d = false /\ is_observe_event d = false)
  \/ (* ... three more disjuncts *).
\end{lstlisting}

\begin{lstlisting}[style=coq]
Theorem primitive_minimality :
  (* Memory: required for Turing completeness *)
  (forall r l, exists params k,
    translate_instruction (INC r l) =
    ITree.bind (trigger (MemoryOp params)) k) /\
  (* Reason: required for oracle extension *)
  (forall params, is_reason_event (LLMCall params) = true /\
                  is_memory_event (LLMCall params) = false) /\
  (* Call: required for external invocation *)
  (forall params, is_call_event (CallMachine params) = true /\
                  is_memory_event (CallMachine params) = false /\
                  is_reason_event (CallMachine params) = false) /\
  (* Event classes are pairwise disjoint *)
  (forall R (d : DirectiveE R),
    (is_memory_event d = true -> is_reason_event d = false) /\
    (is_reason_event d = true -> is_memory_event d = false) /\
    (is_call_event d = true -> is_memory_event d = false) /\
    (* ... all six pairwise implications *) ).
\end{lstlisting}

\begin{theorem}[Strict expressiveness hierarchy]
\label{thm:hierarchy}
The primitive capabilities form a strict chain:
\[
  \code \;<\; \code + \memory \;<\;
  \code + \memory + \reason \;<\;
  \code + \memory + \reason + \mcall
\]
where each level strictly extends the previous: the new primitive emits
events that no combination of the existing primitives can produce.
\end{theorem}

\begin{proof}
Each strict separation combines two facts: (1) the translation at the
new level structurally emits events from the new class (e.g.\
$\texttt{INC}$ emits $\texttt{MemoryOp}$, $\texttt{O\_QUERY}$ emits
$\texttt{LLMCall}$), and (2) event class disjointness ensures the
lower level cannot produce those events.
\end{proof}

\begin{remark}[Expressive minimality]
\label{thm:minimality-informal}
Removing any primitive capability strictly reduces expressive power.
This follows from the hierarchy and disjointness results above.
\end{remark}

% ==================================================================
\section{Subsumption Asymmetry: Structural vs.\ Content Governance}
\label{sec:subsumption}
% ==================================================================

Structural governance mediates effects at the interpretation boundary,
whereas content governance modifies returned values. The safety
predicate $\govsafe(\text{allowed}, t)$ is defined (via \texttt{paco})
as the greatest fixed point of $\texttt{gov\_safeF}$, and the only
constructor that admits visible I/O is:

\begin{lstlisting}[style=coq]
| GS_IOE : (forall x, F true (k x)) ->
    gov_safeF F true (VisF (inr1 e) k)
\end{lstlisting}

\noindent In particular, \texttt{GS\_IOE} requires
$\text{allowed} = \text{true}$. There is no constructor of
$\texttt{gov\_safeF}$ that can derive
$\texttt{gov\_safeF}\;F\;\text{false}\;(\text{VisF}\;(\text{inr1}\;e)\;k)$.

Therefore, any tree whose current observation is a bare I/O event at
permission level $\text{false}$ has no derivation of $\govsafe$.
One proves $\lnot\;\govsafe(\text{false},\;\Vis(\text{inr1}\;e, k))$
by unfolding $\govsafe$ once (via \texttt{punfold}) and inverting on
the resulting $\texttt{gov\_safeF}$ goal: the
$\text{VisF}\;(\text{inr1}\;e)\;k$ case can only be discharged when
$\text{allowed} = \text{true}$.

This argument is independent of any content-level filtering in the
underlying handler: content governance can change values but does not
prevent the occurrence of an unauthorized I/O step. Hence structural
governance strictly subsumes content governance: $\Gov$ enforces
mediation of all effectful directives by construction, while content
filters alone cannot imply effect safety.

\begin{theorem}[Structural subsumes content]
\label{thm:structural-subsumes}
For any handler $h'$ (including content-filtering handlers) and any
program $t$,
\[
  \govsafe(\text{false},\; \interp(\Gov(h'), t)).
\]
\end{theorem}

\begin{theorem}[Content does not subsume structural]
\label{thm:content-not-subsumes}
There exist effectful trees that are content-filtered but not
structurally safe.
\end{theorem}

\begin{corollary}[Subsumption asymmetry]
Structural governance strictly subsumes content governance.
\end{corollary}

% ==================================================================
\section{Capstone Theorem}
\label{sec:capstone}
% ==================================================================

We assemble the properties into a single theorem
\texttt{governed\_cognitive\_completeness}.

\begin{lstlisting}[style=coq]
Theorem governed_cognitive_completeness :
  forall (h : base_handler),
    (* P1: Turing complete AND governed *)
    (forall (p : program) (fuel : nat),
      gov_safe false
        (interp (Gov h) (translate_program p fuel 0)))
    /\
    (* P2: Oracle complete AND governed *)
    (forall (p : oracle_program) (fuel : nat) (pc : label),
      gov_safe false
        (interp (Gov h) (translate_oracle_program p fuel pc)))
    /\
    (* P3: structural_decidable /\ compositional_closure
            /\ halting_nontrivial
            /\ governance_cannot_decide_halting *)
    ...  (* 10 conjuncts; see Rice.v *)
    /\
    (* P4: governed_goal_preservation *)
    ...  (* see GoalDirected.v *)
    /\
    (* P5: cognitive_surjection /\ primitive_minimality
            /\ strict_hierarchy *)
    ...  (* see CognitiveArchitecture.v,
            ExpressiveMinimality.v *)
    /\
    (* P6: structural_subsumes_content
            /\ content_not_subsumes_structural *)
    ...  (* see Subsumption.v *)
    /\
    (* P7: Semantic transparency *)
    (forall (io_h : forall X, IOE X -> X) R
            (t : itree DirectiveE R),
      eutt eq
        (interp (fun X e => Ret (permissive_handler io_h X e))
                (interp (Gov h) t))
        (interp (fun X e => Ret (io_h X e))
                (interp h t))).
\end{lstlisting}

\paragraph{Structure.}
The conjuncts correspond to: governed expressivity (P1--P2), the
decidability boundary and closure (P3), goal preservation (P4),
expressive minimality (P5), subsumption asymmetry (P6), and semantic
transparency (P7). Each conjunct is proved by importing the relevant
module theorem and applying it directly.

% ==================================================================
\section{Discussion and Trust Boundary}
\label{sec:discussion}
% ==================================================================

Our theorems establish properties of a formal model: interaction-tree
programs, a governance wrapper, and a safety predicate. The guarantees
hold relative to a standard trusted computing base: the Rocq kernel
(type-checking the proofs), the Interaction Trees library (providing the
coinductive framework), and any runtime that faithfully implements the
$\Gov$ operator and directive interpretation. This TCB is comparable to
that of verified compilers~\cite{leroy2009compcert} and verified OS
kernels~\cite{klein2009sel4}: formal properties hold of the model, and
the deployment gap is bridged by testing. In our case, the connection
from model to a deployed runtime is validated via a verified interpreter
specification tested against over 70,000 randomly generated directive
sequences with zero disagreements (included in the artifact). Effect
separation for the executor layer is developed
in~\cite{mccann2026purity}.

The key design principle established here is not ``governance by output
restriction,'' but \emph{effect-transparent governance}: constrain the
effect boundary while preserving semantics on permitted executions.

\paragraph{Limitations.}
Semantic transparency is proved relative to executions in which
governance checks succeed; when a check fails, the governed semantics
diverges (\texttt{spin}) rather than simulating the ungoverned run.
The expressivity results use a fuel-indexed register-machine
translation, which captures partiality but abstracts from resource
exhaustion, scheduling, and real-time constraints. Oracle calls are
modeled as directives with unconstrained responses; we do not formalize
probabilistic model behavior, distributional assumptions, or
quantitative uncertainty. Governance predicates are intentionally
structural and total; the development does not attempt to decide
semantic properties such as termination or functional correctness. The
trusted computing base includes the Rocq kernel, the Interaction Trees
library, and any extraction or interpretation layer used to realize the
semantics in a deployed runtime. Accordingly, the theorems are
statements about the abstract architecture and its interpreter, not
end-to-end guarantees about compiled binaries. Finally, transparency is
stated modulo erasure of governance-only events and weak bisimulation
(\texttt{eutt}), reflecting preservation of observable results rather
than literal trace identity.

% ==================================================================
\section{Related Work}
\label{sec:related}
% ==================================================================

\paragraph{Algebraic effect handlers.}
Plotkin and Pretnar~\cite{plotkin2009handlers} introduced handlers for
algebraic effects; Bauer and Pretnar~\cite{bauer2015programming}
developed the programming model. Our governance operator is a handler
wrapper in the same tradition: it interposes on effect interpretation
without modifying the underlying handler. Wadler's monadic
approach~\cite{wadler1995monads} gives the semantic foundation for
interpreting effectful programs through bind; our interaction-tree
development is inherently monadic.

\paragraph{Interaction Trees and coinduction.}
Xia et al.~\cite{xia2020interaction} introduced interaction trees for
representing recursive and impure programs in Rocq; Zakowski et
al.~\cite{zakowski2021itree} applied them to verified LLVM semantics.
Hur et al.~\cite{hur2013power} developed parameterized coinduction
(\texttt{paco}), which underlies both our safety and transparency
proofs.

\paragraph{Verified systems.}
CompCert~\cite{leroy2009compcert} and seL4~\cite{klein2009sel4}
demonstrate that large formal verification efforts yield reliable
systems. Our transparency theorem is structurally analogous to
CompCert's forward simulation: both show that a transformation
preserves observable behavior.

\paragraph{Content governance and alignment.}
RLHF~\cite{ouyang2022training} and Constitutional
AI~\cite{bai2022constitutional} operate at the output level. Our
subsumption theorem formalizes that such methods do not imply
structural effect safety.

\paragraph{Verified AI and agent safety.}
Prior frameworks propose properties to
verify~\cite{seshia2016verified,dalrymple2024guaranteed}; we deliver a
mechanized development with semantic transparency and expressive
minimality.

\paragraph{Own prior work.}
The governance model, safety predicate, and Turing completeness proof
are given in~\cite{mccann2026mechanized}. The structural argument for
why behavioral governance fails at the effect level is developed
in~\cite{mccann2026structural}. Effect separation for the executor
layer is formalized in~\cite{mccann2026purity}. The present paper
extends these with semantic transparency, expressive minimality, and a
decidability boundary with compositional closure.

\paragraph{Cognitive architectures.}
Park et al.~\cite{park2023generative} demonstrated generative agents
with perceive-reflect-act loops; Sumers et
al.~\cite{sumers2024cognitive} surveyed cognitive architectures for
language agents. Our primitive capability taxonomy (compute, memory,
reason, call, observe) is a formalization of capabilities these
architectures require, with machine-checked minimality.

% ==================================================================
\section{Conclusion}
\label{sec:conclusion}
% ==================================================================

We have presented a machine-checked semantics for effect-governed AI
workflow architectures and proved that structural governance can be
imposed without reducing computational expressivity. Our main theorem
shows semantic transparency on permitted executions: governed and
ungoverned interpretations are observationally equivalent modulo
governance-only events. Supporting results establish governed
expressivity (Turing and oracle), a compositional decidability boundary,
expressive minimality of primitive capabilities, and strict subsumption
of content filtering by structural governance.

\paragraph{Artifact availability.}
The complete Rocq development (36 modules, ~12,000 lines, 454 theorems, 0
admitted lemmas) is available at
\url{https://github.com/mashin-live/governance-proofs}.

\bibliographystyle{plainnat}
\bibliography{governed-cognitive-completeness-references}

\end{document}